\documentclass[letterpaper]{article} 
\usepackage{aaai2027}  
\usepackage[hyphens]{url}  
\usepackage{graphicx} 
\usepackage{natbib}  
\usepackage{caption} 
\usepackage{algorithm}
\usepackage{algorithmic}
\usepackage{xspace}
\usepackage{booktabs}
\usepackage{amsmath}
\usepackage{amssymb}
\newtheorem{proposition}{Proposition}
\usepackage{makecell}
\usepackage{multirow}
\usepackage{placeins}
\newcommand{\NA}{N/A}

\usepackage{newfloat}
\usepackage{listings}
\DeclareCaptionStyle{ruled}{labelfont=normalfont,labelsep=colon,strut=off} 
\floatstyle{ruled}
\newfloat{listing}{tb}{lst}{}
\floatname{listing}{Listing}

\nocopyright

\title{Relevant but Incomplete: Referential Dangling as a Paradigm-Level\\
Failure Mode in Hard Prompt Compression}

\author{
    Zhengpei Hu\textsuperscript{\rm 1}\equalcontrib,
    Kai Li\textsuperscript{\rm 2}\equalcontrib,
    Dapeng Fu\textsuperscript{\rm 3},
    Xuechao Zou\textsuperscript{\rm 2},
    Yuanhao Tang\textsuperscript{\rm 1},
    Yue Li\textsuperscript{\rm 1},
    Tengfei Cao\textsuperscript{\rm 1},
    Jianqiang Huang\textsuperscript{\rm 1}\corresponding
}
\affiliations{
    \textsuperscript{\rm 1}School of Computer Technology and Application, Qinghai University\\
    \textsuperscript{\rm 2}Tsinghua University\\
    \textsuperscript{\rm 3}Ant Group Security and Intelligence Laboratory (SIL)
}

\begin{document}
\maketitle

\begin{abstract}
Hard prompt compression reduces long-context inference cost by scoring tokens, sentences, or chunks independently and retaining the highest-scoring units under a budget. We identify a structural failure in this procedure: independent selection can split dependent evidence pairs, retaining one member while deleting the other. When the retained text span contains an answer but the deleted span defines the entity needed to interpret it, we call the result \emph{referential dangling}. At compression ratio $0.30$, \textsc{Beaver}, which ranks coherent chunks using Qwen3-0.6B embeddings, leaves the answer path incomplete in $34$ to $54\%$ of bridge examples across three multi-hop question answering (QA) datasets. The failure is not specific to that implementation: on a shared HotpotQA bridge set, all six hard compressors we test exhibit dangling at rates reaching $60\%$, and every document in LongBench-v2 Single-Document QA contains at least one dangling reference. We then test whether selecting different content at the same budget helps. On dangling examples evaluated with Qwen3-8B, reinserting the missing supporting paragraph and offsetting its tokens by removing paragraphs not annotated as supporting the answer improves accuracy by $29$ to $34$ points ($p<10^{-4}$), recovering at least $88\%$ of the gap to contexts that retain both annotated supporting paragraphs. Stronger answer models do not absorb the loss: on MuSiQue, GPT-5.5 is $8.8$ points less accurate on the compressed contexts than on contexts retaining both supporting paragraphs. Finally, we train a compact classifier to rank sentences the compressor omitted by whether they are needed to interpret retained text, and reinsert the top-ranked candidates without using support annotations at inference. On HotpotQA with Qwen3-8B as the downstream answer model, this automatic restoration improves accuracy by $4.7$ points while changing the compression ratio only from $0.30$ to $0.31$. Hard compressors should optimize relevance and referential completeness.
\end{abstract}

\section{Introduction}

\begin{figure}[t]\centering
\includegraphics[width=\columnwidth]{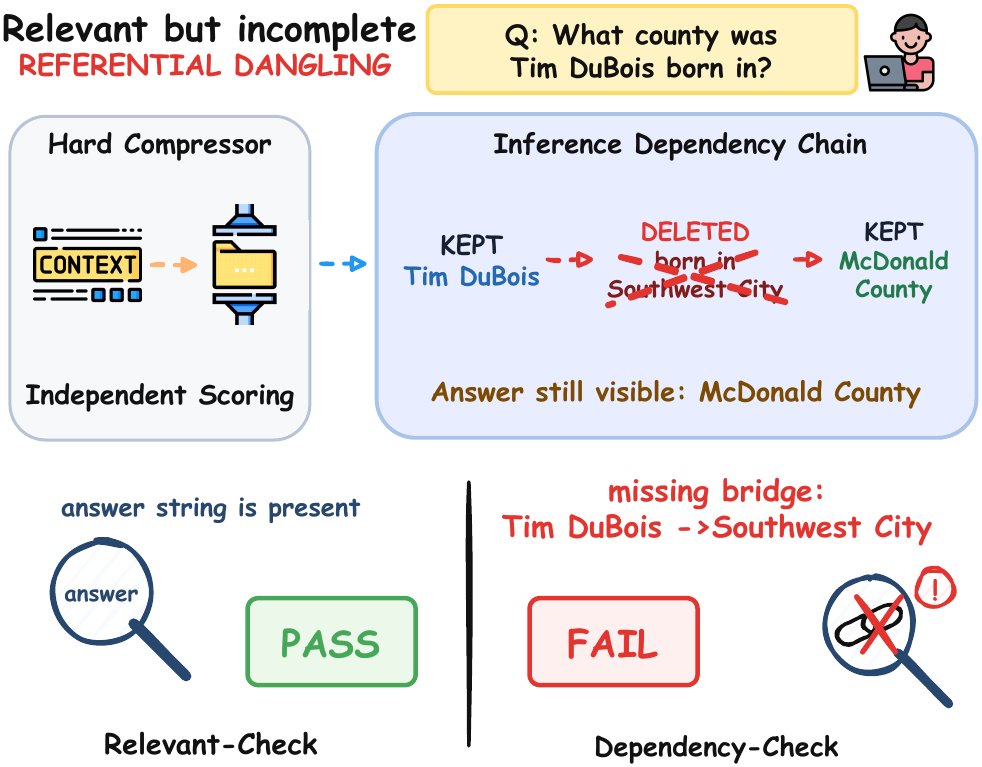}
\caption{Referential dangling with a missing bridge. Independent scoring retains the query subject and the answer string but removes the fact that Tim DuBois was born in Southwest City, leaving the inference chain incomplete.}
\vspace{-15pt}
\label{fig:teaser}\end{figure}

In recent frontier models, the context windows of large language models (LLMs) have expanded from a few thousand tokens to the million-token scale. These expanded windows enable reasoning over entire documents and collections of documents. Yet long contexts remain expensive: the quadratic cost of self-attention inflates prefill latency, and models underuse relevant content when it appears away from prompt boundaries \citep{liu2024lost}. Prompt compression addresses this tension by pruning the input before it reaches the target model while trying to preserve the information needed for the task \citep{jiang2023llmlingua,pan2024llmlingua2,li2023selective}.

Many hard prompt compressors follow a common selection procedure: they score each token, sentence, or chunk for importance and retain the highest scoring fragments within a token budget. These scores may reflect query relevance or perplexity \citep{jiang2024longllmlingua}, embedding similarity \citep{beaver2026}, information content or learned classifier scores \citep{li2023selective,pan2024llmlingua2}, syntactic salience \citep{mao2025partprompt}, or importance derived from attention \citep{zhao2025dac}. Although these methods differ in architecture and supervision, they share the objective of retaining fragments that appear important. This objective assumes that the selected fragments collectively form a usable prompt. However, an important fragment need not be self-contained, because its meaning may depend on an earlier definition, an antecedent, or a bridge fact that the compressor removes independently.

We call this failure \emph{referential dangling} and formalize it in Section~\ref{sec:formulation}. A compressed prompt exhibits referential dangling when it retains a relevant fragment but omits a dependency required for its interpretation, leaving the fragment unusable to the target model. Figure~\ref{fig:teaser} illustrates this mechanism. The compressor retains one fragment mentioning Tim DuBois and another stating that Southwest City is in McDonald County, but deletes the bridge stating that DuBois was born in Southwest City. Although the answer string remains in the compressed prompt, the retained fragments no longer support an inference from the query subject to the answer. This example illustrates why relevance alone does not guarantee a usable compressed context. It remains unclear whether dangling is systematic, fixed-budget reselection recovers accuracy, or restoration can be automated.

We first examine the official implementation of \textsc{Beaver}~\citep{beaver2026}. At compression ratio $0.30$, it leaves answer paths incomplete in $34$ to $54\%$ of bridge examples across HotpotQA~\citep{yang2018hotpotqa}, 2WikiMultiHopQA~\citep{ho2020constructing}, and MuSiQue~\citep{trivedi2022musique}. A human audit finds $95\%$ precision. The diagnosis extends beyond this setting: all six compressors tested exhibit dangling at rates from $32$ to $60\%$ on the shared bridge set, and all $80$ documents in LongBench-v2 Single-Document QA~\citep{bai2025longbenchv2} contain at least one dangling reference. We next conduct an annotation-guided reselection experiment: the omitted supporting paragraph is reinserted, and its token count is offset by removing the lowest-scoring paragraphs not annotated as supporting the answer. On dangling examples evaluated with Qwen3-8B, this intervention improves accuracy by $29$ to $34$ points without increasing the token budget ($p<10^{-4}$) and recovers at least $88\%$ of the accuracy gap between the original compressed contexts and contexts retaining both annotated supporting paragraphs. Finally, we generate candidate sentences from the omitted text and train a compact classifier to rank them according to whether they are needed to interpret retained text. Without using supporting paragraph annotations at inference, the method reinserts the highest ranked candidates. On HotpotQA with Qwen3-8B as the downstream answer model, this automatic restoration improves accuracy by $4.7$ points while changing the compression ratio only from $0.30$ to $0.31$. We released the code, results for individual examples, and the trained model\footnote{\url{https://cslikai.cn/Referential-Dangling/}}.

\section{Related Work}
\label{sec:related}

\paragraph{Prompt and context compression.} Prompt compression reduces long-context inference cost by shortening the input before it reaches the target model. It complements model-side efficiency methods, including elastic subnetworks, compact discrete semantic tokens for multimodal inputs, and efficient encoder--decoder or state-space architectures~\citep{li2024subnetwork,li2026dolphin,li2023tdanet,li2025spmamba}. We study hard compression, which selects tokens, sentences, chunks, or parse nodes under a budget using self-information or perplexity~\citep{li2023selective,jiang2024longllmlingua}, learned token classifiers~\citep{jiang2023llmlingua,pan2024llmlingua2}, embedding similarity~\citep{beaver2026}, syntactic salience~\citep{mao2025partprompt}, attention-based signals~\citep{zhao2025dac}, or training-free sentence and fragment selection~\citep{tang2025perception}. Task-aware selectors add reinforcement-learning rewards or key-information density objectives~\citep{shandilya2025tacorl,lin2025keydensity}, whereas soft or latent approaches use gist tokens, autoencoders, semantic source coding, or learned special-token representations~\citep{mu2023learning,chevalier2023adapting,ge2024incontext,fei2024semantic,li2025500x}. A further line of work compresses hidden activations or KV caches instead of source text~\citep{zhang2024activationbeacon,zhang2023h2o}; because these methods never produce a reduced text sequence, the failure we study does not arise in the same form and we do not evaluate them. Hard methods differ in supervision, granularity, and query access but select fragments by salience or relevance. Existing surveys organize prompt-compression methods by hard versus soft strategies and discuss their downstream adaptations~\citep{li2025surveyprompt}. Comparative studies further show that downstream performance and information preservation vary with the compression method, task, and compression setting~\citep{shandilya2024characterizing,liu2025infopreserve}, while rate-distortion analysis formalizes the budget--performance trade-off and the role of query-aware selection~\citep{nagle2024ratdist}. We study how independent hard selection can retain a fragment but remove its required definition or bridge.

\paragraph{Dependency-preserving selection.} Extractive summarization identified this dependency and built a constraint for it: \citet{durrett2016learning} add anaphoricity constraints to the selection objective, so a sentence may not be extracted unless the text its pronouns depend on is extracted with it. Hard prompt compressors inherit the selection problem without the constraint: they score units by salience or query relevance and keep the highest-scoring units under a budget, with no term rewarding the joint retention of a unit and the text needed to interpret it. Other work approaches the dependency from different angles. Coreference resolution identifies anaphoric links~\citep{lee2017endtoend}, but on complete documents; resolving links in the source does not indicate whether the retained subset stays interpretable once the antecedent is deleted. Faithfulness evaluation asks the converse question, whether a generated summary is supported by its source~\citep{maynez2020faithfulness}, rather than whether the retained source is self-contained. Closest to our setting, \citet{deng2024giststudy} report boundary and information-path failures under gist-based compression. Retrieval-augmented generation retrieves, rewrites, or attends over evidence to favor passages that support generation~\citep{lewis2020retrieval,xu2024recomp,liu2023tcrallm,choi2024r2c}, but scoring by answer support still lets a passage qualify while remaining uninterpretable once its definition is removed. Long-context studies show that accuracy also depends on where relevant information sits and on surrounding distractors~\citep{liu2024lost,shi2023large}, factors that fixed-budget reselection changes and that we therefore control. We measure what the missing constraint costs: how often independent selection splits a dependency pair, and how much downstream accuracy is lost when it does.

\section{Problem Formulation}
\label{sec:formulation}

\paragraph{Referential dangling.}
Let $C=(c_1,\ldots,c_{|C|})$ denote a tokenized context, $Q\in\mathcal{Q}$ a query, and $\tau\in\mathbb{N}$ a token budget. We restrict the formal definition to extractive hard compressors. For any extractive output $A$, let $\operatorname{Pos}_{C}(A)\subseteq\{1,\ldots,|C|\}$ denote the source-token positions emitted in $A$, each emitted at most once and in source order. An extractive compressor $f$ returns $\tilde{C}=f(C,Q,\tau)$ with $|\tilde{C}|\leq\tau<|C|$, where $|\cdot|$ denotes the number of tokens. A compressor that does not use $Q$ is query agnostic. The downstream LLM receives $Q$ and $\tilde{C}$ but not $C$.

A sentence relevant to the task may depend on explicit support elsewhere in the context, and several alternative support paths may be valid. Let $\operatorname{Sent}(C)$ denote the indexed source-sentence occurrences in $C$. For any indexed source fragment $y$, let $\mathcal{I}(y)\subseteq\{1,\ldots,|C|\}$ contain its source-token positions. Let $\mathcal{T}_{C,Q}\subseteq\operatorname{Sent}(C)$ contain the task-relevant sentences whose interpretation or use in answering $Q$ may require such support. We define exact sentence retention by
\[
\operatorname{Ret}_{C}(A)
=\left\{s\in\operatorname{Sent}(C):
\mathcal{I}(s)\subseteq\operatorname{Pos}_{C}(A)\right\}.
\]

For $s\in\mathcal{T}_{C,Q}$ and $D\subseteq\operatorname{Sent}(C)\setminus\{s\}$, let $\operatorname{Suff}_{C,Q}(s,D)$ hold when the source-ordered text formed by $D\cup\{s\}$, together with $Q$, contains all explicit information from $C$ needed to interpret $s$ and use it in an evidence chain for $Q$. Sufficiency is assessed jointly on $D\cup\{s\}$ rather than one dependency edge at a time. Define the complete family of inclusion-minimal sufficient support sets
\[
\begin{aligned}
\mathcal{D}_{C,Q}(s)
=\bigl\{D:\;&D\subseteq\operatorname{Sent}(C)\setminus\{s\},\quad
\operatorname{Suff}_{C,Q}(s,D),\\
&\nexists D'\subsetneq D\ \text{with}\ \operatorname{Suff}_{C,Q}(s,D')\bigr\}.
\end{aligned}
\]
For every $s\in\mathcal{T}_{C,Q}$, we assume $\mathcal{D}_{C,Q}(s)\neq\varnothing$. Because $\operatorname{Sent}(C)$ is finite, every sufficient $D$ contains an inclusion-minimal sufficient subset, so retaining a sufficient support set is equivalent to retaining some member of this minimal family. If $s$ requires no additional support, then $\varnothing\in\mathcal{D}_{C,Q}(s)$. Following the notion of referential completeness in extractive summarization~\citep{durrett2016learning}, we consider only support stated explicitly in $C$ and exclude commonsense inferences, implicit temporal relations, and relations that require external information. Write $\mathcal{R}=\operatorname{Ret}_{C}(\tilde{C})$. The compressed context $\tilde{C}$ exhibits referential dangling if it retains a sentence relevant to the task but retains no sufficient support set in full:
\begin{equation}
\label{eq:dangling}
\exists\,s\in\mathcal{R}\cap\mathcal{T}_{C,Q}
\quad\text{such that}\quad
\forall D\in\mathcal{D}_{C,Q}(s),
\quad D\nsubseteq\mathcal{R}.
\end{equation}
This query-dependent definition distinguishes support in the original context from support retained after compression.

\paragraph{Additive fragment selection.}
We isolate the additive selection rule for hard compressors with nonoverlapping candidate units. Let $\mathcal{F}=\{x_1,\ldots,x_n\}$ be a finite collection of candidate fragments from $C$, with source-token positions $\mathcal{I}(x)$ as above. We assume $\mathcal{I}(x_i)\cap\mathcal{I}(x_j)=\varnothing$ for all $i\neq j$. For any $\mathcal{S}\subseteq\mathcal{F}$, let $\operatorname{Out}_{C}(\mathcal{S})$ be the text formed by emitting, once and in source order, the tokens in $\bigcup_{x\in\mathcal{S}}\mathcal{I}(x)$. Define
\begin{equation}
\begin{aligned}
\operatorname{cost}(\mathcal{S})
&=\big|\operatorname{Out}_{C}(\mathcal{S})\big|
=\sum_{x\in\mathcal{S}}|\mathcal{I}(x)|,\\
\mathcal{R}_{C}(\mathcal{S})
&=\operatorname{Ret}_{C}\!\left(\operatorname{Out}_{C}(\mathcal{S})\right).
\end{aligned}
\end{equation}
This exact formulation covers token-, sentence-, and fixed-chunk selectors with nonoverlapping candidate units whose final selection solves the stated global modular knapsack objective. Selectors with overlapping or hierarchical candidates are not claimed to be exact instances of this objective and are evaluated empirically. For fixed $(C,Q)$, a scorer assigns each fragment a scalar utility
\begin{equation}
u(x)=\sigma(x;C,Q)\in\mathbb{R},
\end{equation}
where $\sigma$ may use the entire context and query. The additive assumption concerns how the final selector combines fragment utilities rather than how those utilities are computed. We model the selected set as a solution to
\begin{equation}
\mathcal{K}
\in\operatorname*{arg\,max}_{\substack{\mathcal{S}\subseteq\mathcal{F}\\
\operatorname{cost}(\mathcal{S})\leq\tau}}
\sum_{x\in\mathcal{S}}u(x).
\end{equation}
Ties are resolved by a fixed deterministic rule.
The resulting compressed text is $\tilde{C}=\operatorname{Out}_{C}(\mathcal{K})$. It avoids referential dangling if and only if
\begin{equation}
\begin{aligned}
\forall s\in\mathcal{R}_{C}(\mathcal{K})\cap\mathcal{T}_{C,Q},
\quad &\exists D\in\mathcal{D}_{C,Q}(s)\\
&\text{such that}\quad D\subseteq\mathcal{R}_{C}(\mathcal{K}).
\end{aligned}
\end{equation}
This condition is the complement of Equation~\eqref{eq:dangling} for the emitted text. The additive objective contains no interaction term that enforces it. The following construction shows that additive selection can omit required support even when the same budget admits a complete alternative.
\begin{proposition}[No guarantee of complete support]
\label{prop:closure}
Let $\mathcal{F}=\operatorname{Sent}(C)=\{x,d,z\}$ consist of three disjoint fragments with equal cost $m\in\mathbb{N}_{>0}$, and assume $\mathcal{R}_{C}(\mathcal{S})=\mathcal{S}$ for every $\mathcal{S}\subseteq\mathcal{F}$. Let $\tau=2m$, $\mathcal{T}_{C,Q}=\{x\}$, and $\mathcal{D}_{C,Q}(x)=\{\{d\}\}$. If $u(x)>u(z)>u(d)\geq0$, then the unique additive maximizer is $\mathcal{K}=\{x,z\}$, whose output dangles even though the feasible selection $\{x,d\}$ has complete support.
\end{proposition}
\noindent\emph{Proof.} The two largest utilities belong to $x$ and $z$, so the unique maximizer under the budget is $\{x,z\}$. Its retained sentence $x$ has the sole sufficient support set $\{d\}$, which is absent. The alternative $\{x,d\}$ has cost $2m$ and retains this support set, so it is feasible and complete. \hfill$\square$

The proposition is an existence result. It shows that the additive objective alone provides no guarantee for arbitrary utilities, but it does not claim that every compressed context will dangle or that a complete alternative always fits the same budget. The exact class includes context-aware~\citep{liskavets2025prompt} or reinforcement-learned scores~\citep{jung2024discrete} only when the final selector solves the stated fixed modular objective without a support constraint. Section~\ref{sec:diagnosis} evaluates recurrence empirically for selectors both inside and outside this class.

\section{Empirical Diagnosis of Referential Dangling}
\label{sec:diagnosis}

The formulation in Section~\ref{sec:formulation} motivates two empirical analyses. We first use \textsc{Beaver} to measure the prevalence of referential dangling and its variation with compression ratio, annotated hop count, and reference distance. We then compare six compressors to test whether the diagnostic recurs across scoring signals and output granularities.

\paragraph{Measurement protocol.}
Because support annotations are incomplete and exact sentence retention is not uniformly available, we use directional content-word coverage. Let $\operatorname{CW}(y)$ be the normalized content-word multiset of unit $y$, and let $\mathcal{U}(A)$ contain the retained output units. When $|\operatorname{CW}(v)|>0$ and $\mathcal{U}(A)\neq\varnothing$, define
\[
\operatorname{cov}(v,A)
=\max_{a\in\mathcal{U}(A)}
\frac{|\operatorname{CW}(v)\cap\operatorname{CW}(a)|}
     {|\operatorname{CW}(v)|}.
\]
Set $\operatorname{cov}(v,A)=0$ when $|\operatorname{CW}(v)|=0$ or $\mathcal{U}(A)=\varnothing$. Span-, sentence-, and chunk-level outputs use native retained units; token-level outputs regroup retained tokens by source sentence. Intersections count multiplicity, and the maximum is per output unit, not over their union. For $\theta\in(0,1]$, write $\operatorname{Keep}_{\theta}(v,A)$ when $\operatorname{cov}(v,A)\geq\theta$. We use $\theta=0.5$, with a sweep in Appendix~\ref{app:robustness}. This protocol approximates Equation~\eqref{eq:dangling} rather than evaluating it exactly.

In bridge questions, the definition paragraph introduces a bridge entity and the answer paragraph refers to it. At each evaluation granularity, we apply the same directional event to the paired answer and definition units: an example is answer-path dangling when $\operatorname{Keep}_{\theta}$ holds for the answer unit but not for the definition unit. The prevalence analysis uses paragraphs, whereas the cross-compressor analysis uses the corresponding source sentences. This test does not enumerate alternative support paths.

Among bridge examples whose answer paragraph satisfies $\operatorname{Keep}_{\theta}$, $\rho_{\mathrm{d}}$ is the fraction whose paired definition paragraph does not satisfy $\operatorname{Keep}_{\theta}$. The complete evidence retention rate $\rho_{\mathrm{e}}$ is the fraction of all evaluated examples retaining every annotated evidence paragraph. LongBench-v2 Single-Document QA uses a separate first-mention diagnostic.

\subsection{Referential Dangling under \textsc{Beaver}}
\label{sec:prevalence}

\begin{figure*}[t]\centering
\includegraphics[width=0.9\textwidth]{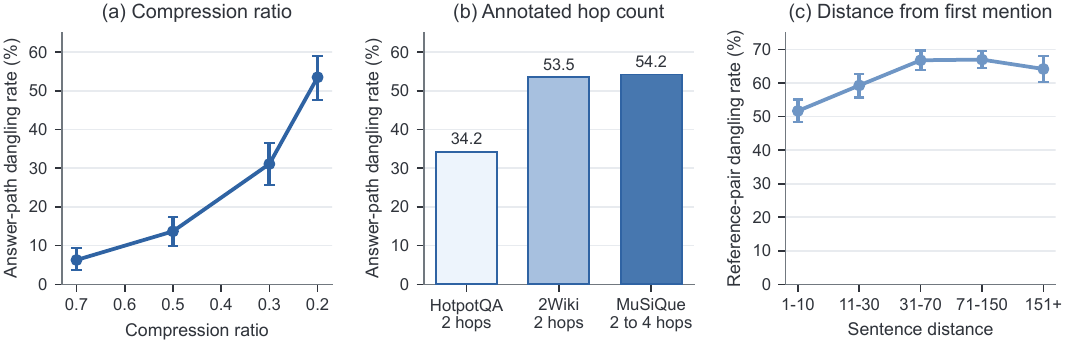}
\caption{Referential dangling under \textsc{Beaver}. Panel (a) reports $\rho_{\mathrm{d}}$ across compression ratios on HotpotQA ($n{=}269$ to $300$ per point, including partial paragraph retention). Panel (b) reports $\rho_{\mathrm{d}}$ by annotated hop count on HotpotQA ($n{=}234$), 2WikiMultiHopQA ($n{=}241$), and MuSiQue ($n{=}286$). Panel (c) reports dangling rates for $4{,}649$ reference pairs in LongBench-v2 Single-Document QA by sentence distance from first mention to later reference. Error bars are bootstrap $95\%$ confidence intervals.}
\vspace{-5pt}
\label{fig:dose}\end{figure*}

\paragraph{Setup.} We use the official \textsc{Beaver} implementation~\citep{beaver2026} with Qwen3-0.6B embeddings~\citep{yang2025qwen3,qwen3embedding2025}. We evaluate HotpotQA~\citep{yang2018hotpotqa}, 2WikiMultiHopQA~\citep{ho2020constructing}, and MuSiQue~\citep{trivedi2022musique}. HotpotQA and 2WikiMultiHopQA use contexts of about $40$ paragraphs and $5.7$k tokens, formed by mixing annotated supporting paragraphs with sampled distractors; MuSiQue uses its native $20$-paragraph contexts. We define $r=|\tilde C|/|C|$ and use binary search to reach each target ratio.

\paragraph{Prevalence.} At $r=0.30$, $\rho_{\mathrm{d}}$ ranges from $34.2\%$ to $54.2\%$ across the three datasets, while $\rho_{\mathrm{e}}$ ranges from $27.0\%$ to $61.0\%$ (Table~\ref{tab:prevalence}). MuSiQue has the highest $\rho_{\mathrm{d}}$ and the lowest $\rho_{\mathrm{e}}$. Figure~\ref{fig:dose}a uses a broader HotpotQA sample that includes partial paragraph retention ($n{=}293$ at $r=0.30$) and reports $31.1\%$, compared with $34.2\%$ for the $234$-example table sample. In a manual audit, the omitted paragraph was required in $38$ of $40$ flagged examples, yielding $95\%$ precision. Appendix~\ref{app:gallery} reports the dependency categories, false positives, and qualitative examples.

\begin{table}[t]
\centering
\setlength{\tabcolsep}{5pt}
\begin{tabular}{lccc}
\toprule
Dataset & Hops & $\rho_{\mathrm{d}}$ (\%) & $\rho_{\mathrm{e}}$ (\%) \\
\midrule
HotpotQA & 2 & 34.2 & 61.0 \\
2WikiMultiHopQA & 2 & 53.5 & 30.7 \\
MuSiQue & 2 to 4 & 54.2 & 27.0 \\
\bottomrule
\end{tabular}
\caption{Referential dangling and complete evidence retention under \textsc{Beaver} at $r=0.30$.}
\vspace{-10pt}
\label{tab:prevalence}
\end{table}

\paragraph{Long documents.} LongBench-v2 Single-Document QA~\citep{bai2025longbenchv2} lacks annotations of supporting evidence and multi-hop structure, so we use a separate first-mention diagnostic. A later reference dangles when its entity's first-mention sentence is omitted while the reference sentence survives. Across $80$ documents at $r=0.30$, the mean per-document rate is $30.5\%$, and every document contains at least one dangling reference. Rates range from $25\%$ to $37\%$ across seven subdomains (Appendix~\ref{app:longdoc}). Because this diagnostic uses retained sentences as its denominator rather than bridge examples with a retained answer paragraph, it is not directly comparable with $\rho_{\mathrm{d}}$.

\paragraph{Variation across measured conditions.} On the broader HotpotQA sample in Figure~\ref{fig:dose}a, $\rho_{\mathrm{d}}$ rises from $6.3\%$ at $r=0.70$ to $53.5\%$ at $r=0.20$, with nonoverlapping bootstrap $95\%$ confidence intervals between adjacent operating points. This endpoint is distinct from the $53.5\%$ reported for 2WikiMultiHopQA in Table~\ref{tab:prevalence}. Dataset groups with larger annotated hop counts also have larger $\rho_{\mathrm{d}}$ values (Figure~\ref{fig:dose}b), but dataset construction differs, so this comparison does not isolate reasoning depth. On LongBench-v2 Single-Document QA, the pair-level dangling rate is $51.7\%$ for distances of $1$ to $10$ sentences, $67.0\%$ for $71$ to $150$, and $64.2\%$ beyond $150$. Dangling pairs have a mean distance of $89$ sentences, compared with $73$ for complete pairs.

\subsection{Comparison Across Six Scoring Signals}
\label{sec:paradigm}

The comparison includes six hard compressors that assign importance using different signals: embedding similarity with \textsc{Beaver}~\citep{beaver2026}, syntactic parse structure with PartPrompt~\citep{mao2025partprompt}, self-information with Selective-Context~\citep{li2023selective}, a learned token classifier with LLMLingua-2~\citep{pan2024llmlingua2}, perplexity with LongLLMLingua~\citep{jiang2024longllmlingua}, and importance derived from attention with DAC~\citep{zhao2025dac}. These methods span chunk, parse node, sentence, and token outputs. \textsc{Beaver} and LongLLMLingua use the query, whereas the other four methods do not. The procedures differ, but none explicitly constrains joint dependency retention.

All six compressors process the same $184$ HotpotQA bridge examples at compression ratio $0.30$. We apply the sentence-level predicate $\operatorname{Keep}_{0.5}$ defined above to outputs at every granularity. Under this criterion, \textsc{Beaver} has a dangling rate of $32.1\%$ on the shared set. Table~\ref{tab:prevalence} reports $34.2\%$ under the paragraph-level predicate on a different $234$-example sample, so the two values are not a controlled comparison of the predicates. Appendix~\ref{app:robustness} sweeps the threshold from $0.3$ to $0.7$; Appendix~\ref{app:compressor-configs} details each compressor's implementation.

\begin{figure}[t]\centering
\includegraphics[width=0.95\columnwidth]{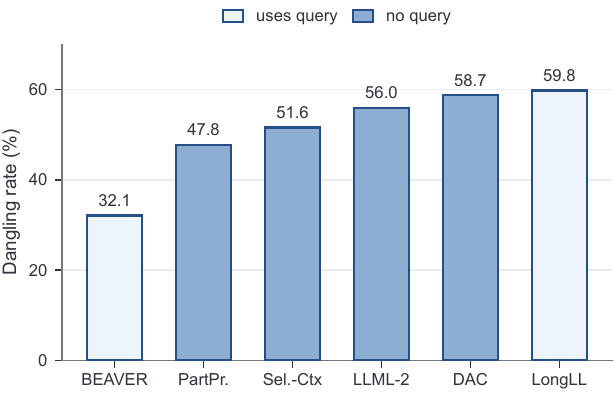}
\caption{Dangling rates for six compressors on the shared HotpotQA bridge set ($n{=}184$) at compression ratio $0.30$. PartPr. denotes PartPrompt, Sel.-Ctx denotes Selective-Context, LLML-2 denotes LLMLingua-2, and LongLL denotes LongLLMLingua. All outputs are evaluated using the content-word overlap criterion with threshold $0.5$. Light bars denote methods that use the query, and darker bars denote methods that do not.}
\label{fig:paradigm}\vspace{-10pt}\end{figure}

\paragraph{Results across scoring signals.} Figure~\ref{fig:paradigm} reports rates from $32.1\%$ for \textsc{Beaver} to $59.8\%$ for LongLLMLingua. Referential dangling is therefore observed beyond selection based on embeddings. The two methods that use the query, \textsc{Beaver} and LongLLMLingua, have rates at opposite ends of the measured range, so query access alone does not remove the diagnostic event. PartPrompt has a rate of $47.8\%$ on the shared set, indicating that the event also occurs under hierarchical syntactic selection in this setting. These comparisons extend the diagnosis beyond the exact additive model without implying that every evaluated selector instantiates it.

\paragraph{Definition salience.} We next examine whether omitted definitions simply receive low importance scores. For each bridge example, we measure the isolated salience percentile of the answer sentence and the definition sentence among all sentences using \textsc{Beaver} sentence similarity, which uses the query, and self-information, which does not. Figure~\ref{fig:mechanism} reports the results for $180$ examples. Under \textsc{Beaver} similarity, the mean percentile is $92.8\%$ for definitions and $84.1\%$ for answer sentences. Under self-information, the corresponding values are $46.7\%$ and $54.3\%$. Low isolated salience is therefore insufficient to explain the omitted definitions, although these measurements do not identify the reason for each individual selection decision.

\begin{figure}[t]\centering
\includegraphics[width=0.9\columnwidth]{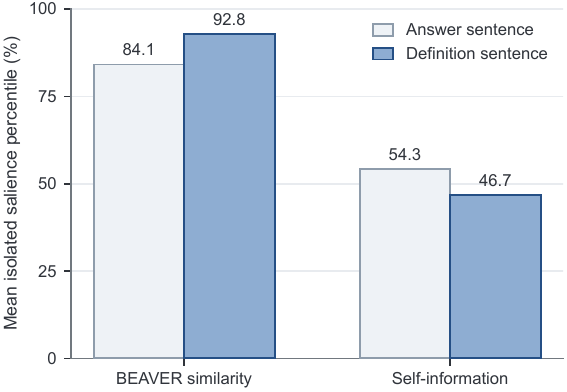}
\caption{Mean salience percentiles (\%) for answer and definition sentences among all sentences in $180$ bridge examples. \textsc{Beaver} similarity is query-aware; self-information is not.}
\label{fig:mechanism}\vspace{-10pt}\end{figure}

\paragraph{Overlap of affected examples.} We compare the dangling case sets using Jaccard similarity to determine whether the compressors fail on the same examples. Table~\ref{tab:cross_compressor_overlap} reports a mean pairwise similarity of $0.38$, compared with $0.34$ expected under independent occurrence at the observed rates. Only $5.4\%$ of examples dangle under all six compressors, whereas $95\%$ dangle under at least one. The affected sets therefore vary across scoring methods, a pattern consistent with selection that varies by method rather than a single fixed pool of difficult examples. These measurements quantify prevalence under the diagnostic but do not determine whether changing the selected content at the same budget improves downstream accuracy, which we evaluate next.

\begin{table}[t]
\centering
\footnotesize
\setlength{\tabcolsep}{2pt}
\begin{tabular}{lcccccc}
\toprule
& BEAVER & PartPr. & Sel.-Ctx & LLML-2 & DAC & LongLL \\
\midrule
\makecell[l]{Dangling\\rate (\%)} & 32.1 & 47.8 & 51.6 & 56.0 & 58.7 & 59.8 \\
\midrule
BEAVER & \NA & 0.36 & 0.23 & 0.30 & 0.29 & 0.32 \\
PartPr. & 0.36 & \NA & 0.36 & 0.44 & 0.44 & 0.37 \\
Sel.-Ctx & 0.23 & 0.36 & \NA & 0.39 & 0.35 & 0.51 \\
LLML-2 & 0.30 & 0.44 & 0.39 & \NA & 0.47 & 0.45 \\
DAC & 0.29 & 0.44 & 0.35 & 0.47 & \NA & 0.48 \\
LongLL & 0.32 & 0.37 & 0.51 & 0.45 & 0.48 & \NA \\
\bottomrule
\end{tabular}
\caption{Pairwise Jaccard similarities between dangling case sets on the shared HotpotQA bridge set ($n{=}184$) at compression ratio $0.30$. The first row reports the dangling rate of each compressor. Abbreviations match Figure~\ref{fig:paradigm}.}
\label{tab:cross_compressor_overlap}
\vspace{-10pt}
\end{table}

\section{Fixed-Budget Content Reselection}
\label{sec:reselection}

The preceding diagnostics show that compression often separates annotated support pairs, but prevalence alone does not determine whether selecting different content improves downstream accuracy without increasing the token budget. We test this by restoring the omitted supporting paragraph and removing at least as many tokens from paragraphs that are not annotated as supporting the answer.

\paragraph{Comparison protocol.} Panel 1 of Table~\ref{tab:reselection} evaluates three content conditions on the dangling subsets. Base is the original compressed context. Reselected reinserts the omitted supporting paragraph and removes the lowest scoring paragraphs not annotated as supporting the answer until at least the same number of tokens has been removed. Full support retains both annotated supporting paragraphs and provides a reference for complete evidence retention. Base and Reselected form the fixed-budget comparison, while Full support provides the reference used to quantify the accuracy recovered through reselection. Panel 2 separately compares Base with Full support on $200$ HotpotQA examples and does not evaluate Reselected. We conduct these comparisons with \textsc{Beaver} because its coherent chunk output permits paragraph replacement while leaving the remaining retained text unchanged. Applying the same operation to token-level output would require inserting whole paragraphs into fragmented text and would not yield a directly comparable setting. Appendix D specifies the construction for each dataset.

\begin{table*}[t]
\centering
\begin{tabular}{llcccc}
\toprule
\multicolumn{6}{l}{\emph{Panel 1: dangling subsets evaluated with Qwen3-8B}} \\
Dataset & Downstream LLM & Base & Reselected & Full support & McNemar $p$ \\
\midrule
HotpotQA & Qwen3-8B & $0.287$ & $0.575$ & $0.600$ & $1.6\times10^{-6}$ \\
2WikiMultiHopQA & Qwen3-8B & $0.097$ & $0.403$ & $0.444$ & $1.1\times10^{-5}$ \\
MuSiQue & Qwen3-8B & $0.147$ & $0.490$ & $0.471$ & $3.1\times10^{-8}$ \\
\midrule
\multicolumn{6}{l}{\emph{Panel 2: a $200$ example HotpotQA evaluation set with four downstream LLMs}} \\
Dataset & Downstream LLM & Base & \multicolumn{2}{c}{Full support} & McNemar $p$ \\
\midrule
HotpotQA & Qwen3-8B & $0.535$ & \multicolumn{2}{c}{$0.615$} & $0.001$ \\
& Qwen3-4B & $0.500$ & \multicolumn{2}{c}{$0.585$} & $0.002$ \\
& Llama-3.1-8B & $0.575$ & \multicolumn{2}{c}{$0.660$} & $0.004$ \\
& Mistral-7B & $0.455$ & \multicolumn{2}{c}{$0.545$} & $0.0005$ \\
\bottomrule
\end{tabular}
\caption{Answer accuracy with base contexts produced by \textsc{Beaver} at target compression ratio $0.30$. Panel 1 uses the dangling subsets of HotpotQA ($n{=}80$), 2WikiMultiHopQA ($n{=}72$), and MuSiQue ($n{=}102$), with McNemar $p$ values comparing Base and Reselected. Panel 2 uses a separate set of $200$ HotpotQA examples, with McNemar $p$ values comparing Base and Full support.}
\label{tab:reselection}
\end{table*}

We evaluate GPT-5.5~\citep{openai2026models} and GLM-5.2~\citep{zhipu2026glm52} as proprietary downstream models. Their API identifiers are listed in Table~\ref{tab:checkpoints}.

\begin{table}[t]
\centering
\footnotesize
\setlength{\tabcolsep}{3pt}
\begin{tabular}{llcccc}
\toprule
Model & Dataset & $n$ & Base & Full support & McNemar $p$ \\
\midrule
GPT-5.5 & HotpotQA & $184$ & $0.913$ & $0.913$ & $1.0$ \\
GPT-5.5 & MuSiQue & $102$ & $0.775$ & $0.863$ & $0.011$ \\
GLM-5.2 & MuSiQue & $95$ & $0.695$ & $0.937$ & $2.4\times10^{-7}$ \\
\bottomrule
\end{tabular}
\caption{Answer accuracy of proprietary models under Base and Full support, with base contexts produced by \textsc{Beaver} at target compression ratio $0.30$. HotpotQA uses the full shared bridge set, while MuSiQue uses the dangling subset. GLM-5.2 returned answers for $95$ of the $102$ MuSiQue contexts because of API timeouts.}
\label{tab:closed}
\vspace{-6pt}
\end{table}

\paragraph{Accuracy on dangling examples.} The dangling subsets are defined by paragraph retention and do not depend on whether the downstream model answers correctly. Panel 1 of Table~\ref{tab:reselection} shows that Reselected improves accuracy over Base by $28.8$ to $34.3$ points, with $p<10^{-4}$ on all three datasets. It recovers $92\%$ and $88\%$ of the difference between Base and Full support on HotpotQA and 2WikiMultiHopQA, respectively, and exceeds Full support by $1.9$ points on MuSiQue. Because Reselected and Full support remove different paragraphs that are not annotated as supporting the answer, these comparisons measure the joint change in retained content rather than the isolated contribution of the reinserted paragraph. The reselected contexts are shorter than the base contexts on average, as reported in Appendix D, so their gains cannot be explained by a larger token budget. The protocol relies on annotated supporting paragraphs, and we apply it to all three multi-hop QA datasets used in the prevalence analysis.

\paragraph{Results across downstream models.} We next compare Base and Full support across Qwen3, Llama-3.1, and Mistral models~\citep{yang2025qwen3,dubey2024llama3,jiang2023mistral}. Exact checkpoint and API identifiers are listed in Table~\ref{tab:checkpoints}. Panel 2 of Table~\ref{tab:reselection} reports gains of $8.0$ to $9.0$ accuracy points for Full support, with $p<0.01$ for all four models. These comparisons remain significant after Holm correction at $\alpha{=}0.05$. The gains for Qwen3-4B and Qwen3-8B are similar, at $8.5$ and $8.0$ points, respectively. Additional comparisons with proprietary models are reported in Table~\ref{tab:closed}. On HotpotQA, GPT-5.5 has the same aggregate accuracy under both conditions because restoring both supporting paragraphs produces four fixes and four breaks. On MuSiQue, Full support improves accuracy by $8.8$ points for GPT-5.5 and $24.2$ points for GLM-5.2. Full support improves accuracy for all four evaluated open-weight models on HotpotQA and both proprietary models on MuSiQue, while GPT-5.5 shows no aggregate difference on HotpotQA.

\section{Automatic Context Restoration}
\label{sec:method}

The fixed-budget comparison relies on annotated supporting paragraphs to determine which content to restore. We therefore test whether omitted supporting sentences can be selected automatically at inference with only a small increase in the token budget. We treat the resulting pipeline as a diagnostic of targeted sentence restoration rather than a complete compression system. We evaluate it on \textsc{Beaver} outputs because their coherent blocks containing multiple sentences permit controlled insertion at sentence boundaries while leaving the remaining compressed context unchanged.

\paragraph{Restoration pipeline and training.} A candidate generator first collects sentences from the omitted context that may support retained text. A \texttt{bert-base-uncased} classifier~\citep{devlin2019bert} then ranks the candidates using a retained sentence, a candidate sentence, and the question, separated by \texttt{[SEP]} tokens. Training pairs are constructed from the HotpotQA training split. Positive pairs consist of a retained sentence and an omitted annotated supporting sentence that share an entity. Pairs that share an entity but do not meet this positive criterion serve as hard negatives, whereas pairs without a shared entity serve as easy negatives. After negative downsampling, the training set contains $7{,}565$ pairs, of which $40\%$ are positive. Pairs are split by example identifier into classifier training and development partitions, and the source contexts are disjoint from the downstream evaluation contexts. We fine-tune the classifier for three epochs with AdamW using a learning rate of $2\times10^{-5}$, a batch size of $32$, and a maximum sequence length of $256$, and retain the checkpoint with the highest development F1. At inference, the $K$ candidates with the highest classifier scores are reinserted. We use $K{=}3$ in the main analyses; Appendix~\ref{app:detector} reports the $K$ sweep and remaining implementation details.

\paragraph{Downstream accuracy.} With first-mention candidates, the restoration pipeline improves Qwen3-8B accuracy by $4.7$ points ($p{=}0.022$) while changing the compression ratio from $0.30$ to $0.31$. The same procedure improves Mistral-7B accuracy by $6.5$ points ($p{=}0.012$). On Llama-3.1-8B, the gain ranges from $1.3$ to $2.3$ points across candidate sources and is not statistically significant. Subsequent analyses use Qwen3-8B; Appendix G reports all downstream-model results.

\paragraph{Candidate sources.} Table~\ref{tab:ablation} compares candidate sources while holding the classifier fixed. First-mention candidates yield a gain of $4.7$ points on HotpotQA but only $0.5$ points on 2WikiMultiHopQA. Combining first-mention candidates with embedding retrieval gives point estimates of $4.5$ and $5.5$ points, respectively. In a diagnostic condition that includes the annotated supporting sentence in the candidate set, the gain reaches $8.0$ points ($p{=}0.008$) while adding only $0.4$ sentences on average. The larger point estimate suggests that candidate construction limits the current pipeline. Appendix G further compares cases with successful and unsuccessful restoration.

\begin{table}[t]
\centering
\footnotesize
\setlength{\tabcolsep}{1.5pt}
\begin{tabular}{lcc}
\toprule
Candidate source & HotpotQA & 2WikiMultiHopQA \\
\midrule
First mention & +4.7 ($p{=}0.022$) & \makecell{+0.5\\(not significant)} \\
All mentions & +4.5 ($p{=}0.15$) & +4.0 ($p{=}0.20$) \\
Hybrid & +4.5 ($p{=}0.12$) & +5.5 ($p{=}0.063$) \\
\midrule
Annotated support included & \textbf{+8.0} ($p{=}0.008$) & \NA \\
\bottomrule
\end{tabular}
\caption{Changes in answer accuracy, in percentage points relative to Base, for candidate sources with a fixed classifier and Qwen3-8B ($K{=}3$). Hybrid augments first-mention candidates with embedding retrieval, and the final row includes the annotated supporting sentence in the candidate set.}
\label{tab:ablation}
\vspace{-10pt}
\end{table}

\paragraph{Matched addition control.} Adding the same number of randomly selected sentences improves accuracy by $2.0$ points, compared with $4.7$ points for targeted restoration (Appendix G). The direct contrast is suggestive but not significant at this sample size ($p{=}0.15$); it therefore does not establish an advantage over random insertion.

\paragraph{Transfer across compressors.} Applying the same restoration configuration to three additional compressor outputs yields smaller gains that are not statistically significant. Because the configuration is calibrated on \textsc{Beaver} and HotpotQA and output granularity varies, transfer to other compressors remains unresolved (Appendix H).

\section{Conclusion}

We identify referential dangling, in which independent hard compression retains task-relevant text but removes support required for interpretation. At $r=0.30$, paragraph-level dangling occurs in $34\%$ to $54\%$ of bridge examples with a retained answer paragraph under \textsc{Beaver} across three multi-hop QA datasets. The event also recurs across six compressors on HotpotQA, and a separate LongBench-v2 Single-Document QA diagnostic flags every evaluated document. On affected \textsc{Beaver} examples, fixed-budget reselection improves Qwen3-8B accuracy by $29$ to $34$ points, and automatic restoration adds $4.7$ points on HotpotQA while changing the compression ratio from $0.30$ to $0.31$. These findings motivate dependency-preserving selection, although transfer beyond \textsc{Beaver} and annotated-support QA remains unresolved.

\bibliography{aaai2027}

@article{liu2024lost,
  title={Lost in the middle: How language models use long contexts},
  author={Liu, Nelson F and Lin, Kevin and Hewitt, John and Paranjape, Ashwin and Bevilacqua, Michele and Petroni, Fabio and Liang, Percy},
  journal={Transactions of the association for computational linguistics},
  volume={12},
  pages={157--173},
  year={2024}
}

@inproceedings{shi2023large,
  title={Large language models can be easily distracted by irrelevant context},
  author={Shi, Freda and Chen, Xinyun and Misra, Kanishka and Scales, Nathan and Dohan, David and Chi, Ed H and Sch{\"a}rli, Nathanael and Zhou, Denny},
  booktitle={International Conference on Machine Learning},
  pages={31210--31227},
  year={2023},
  organization={PMLR}
}

@inproceedings{jiang2023llmlingua,
  title={Llmlingua: Compressing prompts for accelerated inference of large language models},
  author={Jiang, Huiqiang and Wu, Qianhui and Lin, Chin-Yew and Yang, Yuqing and Qiu, Lili},
  booktitle={Proceedings of the 2023 conference on empirical methods in natural language processing},
  pages={13358--13376},
  year={2023}
}

@inproceedings{pan2024llmlingua2,
  title={Llmlingua-2: Data distillation for efficient and faithful task-agnostic prompt compression},
  author={Pan, Zhuoshi and Wu, Qianhui and Jiang, Huiqiang and Xia, Menglin and Luo, Xufang and Zhang, Jue and Lin, Qingwei and R{\"u}hle, Victor and Yang, Yuqing and Lin, Chin-Yew and others},
  booktitle={Findings of the Association for Computational Linguistics: ACL 2024},
  pages={963--981},
  year={2024}
}

@inproceedings{jiang2024longllmlingua,
  title={Longllmlingua: Accelerating and enhancing llms in long context scenarios via prompt compression},
  author={Jiang, Huiqiang and Wu, Qianhui and Luo, Xufang and Li, Dongsheng and Lin, Chin-Yew and Yang, Yuqing and Qiu, Lili},
  booktitle={Proceedings of the 62nd Annual Meeting of the Association for Computational Linguistics (Volume 1: Long Papers)},
  pages={1658--1677},
  year={2024}
}

@inproceedings{li2023selective,
  title={Compressing context to enhance inference efficiency of large language models},
  author={Li, Yucheng and Dong, Bo and Guerin, Frank and Lin, Chenghua},
  booktitle={Proceedings of the 2023 conference on empirical methods in natural language processing},
  pages={6342--6353},
  year={2023}
}

@misc{beaver2026,
  title={BEAVER: A Training-Free Hierarchical Prompt Compression Method via Structure-Aware Page Selection},
  author={Hu, Zhengpei and Li, Kai and Fu, Dapeng and Zeng, Chang and Li, Yue and Tang, Yuanhao and Huang, Jianqiang},
  journal={arXiv preprint arXiv:2603.19635},
  year={2026}
}

@article{mao2025partprompt,
  title={Parse trees guided LLM prompt compression},
  author={Mao, Wenhao and Hou, Chengbin and Zhang, Tianyu and Lin, Xinyu and Tang, Ke and Lv, Hairong},
  journal={IEEE Transactions on Pattern Analysis and Machine Intelligence},
  year={2025},
  publisher={IEEE}
}

@inproceedings{zhao2025dac,
  title={DAC: A dynamic attention-aware approach for task-agnostic prompt compression},
  author={Zhao, Yi and Li, Zuchao and Zhao, Hai and Qi, Baoyuan and Guoming, Liu},
  booktitle={Proceedings of the 63rd Annual Meeting of the Association for Computational Linguistics (Volume 1: Long Papers)},
  pages={19395--19407},
  year={2025}
}

@inproceedings{xu2024recomp,
  title={Recomp: Improving retrieval-augmented lms with context compression and selective augmentation},
  author={Xu, Fangyuan and Shi, Weijia and Choi, Eunsol},
  booktitle={International Conference on Learning Representations},
  volume={2024},
  pages={43478--43502},
  year={2024}
}

@article{jung2024discrete,
  title={Discrete prompt compression with reinforcement learning},
  author={Jung, Hoyoun and Kim, Kyung-Joong},
  journal={IEEE Access},
  volume={12},
  pages={72578--72587},
  year={2024},
  publisher={IEEE}
}

@inproceedings{liskavets2025prompt,
  title={Prompt compression with context-aware sentence encoding for fast and improved llm inference},
  author={Liskavets, Barys and Ushakov, Maxim and Roy, Shuvendu and Klibanov, Mark and Etemad, Ali and Luke, Shane K},
  booktitle={Proceedings of the AAAI Conference on Artificial Intelligence},
  volume={39},
  number={23},
  pages={24595--24604},
  year={2025}
}

@article{mu2023learning,
  title={Learning to compress prompts with gist tokens},
  author={Mu, Jesse and Li, Xiang and Goodman, Noah},
  journal={Advances in Neural Information Processing Systems},
  volume={36},
  pages={19327--19352},
  year={2023}
}

@inproceedings{chevalier2023adapting,
  title={Adapting language models to compress contexts},
  author={Chevalier, Alexis and Wettig, Alexander and Ajith, Anirudh and Chen, Danqi},
  booktitle={Proceedings of the 2023 Conference on Empirical Methods in Natural Language Processing},
  pages={3829--3846},
  year={2023}
}

@misc{ge2024incontext,
  title={In-context autoencoder for context compression in a large language model},
  author={Ge, Tao and Hu, Jing and Wang, Lei and Wang, Xun and Chen, Si-Qing and Wei, Furu},
  journal={arXiv preprint arXiv:2307.06945},
  year={2023}
}

@inproceedings{li2025surveyprompt,
  title={Prompt compression for large language models: A survey},
  author={Li, Zongqian and Liu, Yinhong and Su, Yixuan and Collier, Nigel},
  booktitle={Proceedings of the 2025 Conference of the Nations of the Americas Chapter of the Association for Computational Linguistics: Human Language Technologies (Volume 1: Long Papers)},
  pages={7182--7195},
  year={2025}
}

@inproceedings{li2024subnetwork,
  title={Subnetwork-to-go: Elastic neural network with dynamic training and customizable inference},
  author={Li, Kai and Luo, Yi},
  booktitle={ICASSP 2024-2024 IEEE International Conference on Acoustics, Speech and Signal Processing (ICASSP)},
  pages={6775--6779},
  year={2024},
  organization={IEEE}
}

@misc{li2023tdanet,
  title={An efficient encoder-decoder architecture with top-down attention for speech separation},
  author={Li, Kai and Yang, Runxuan and Hu, Xiaolin},
  journal={arXiv preprint arXiv:2209.15200},
  year={2022}
}

@inproceedings{li2025spmamba,
  title={SPMamba: Leveraging Long-Sequence Modeling with State Space Models for Speech Separation},
  author={Li, Kai and Chen, Guo and Yang, Runxuan and Hu, Xiaolin},
  booktitle={2025 IEEE International Conference on Multimedia and Expo (ICME)},
  pages={1--6},
  year={2025},
  organization={IEEE}
}

@misc{li2026dolphin,
  title={Efficient Audio-Visual Speech Separation with Discrete Lip Semantics and Multi-Scale Global-Local Attention},
  author={Li, Kai and Gao, Kejun and Hu, Xiaolin},
  journal={arXiv preprint arXiv:2509.23610},
  year={2025}
}

@inproceedings{yang2018hotpotqa,
  title={HotpotQA: A dataset for diverse, explainable multi-hop question answering},
  author={Yang, Zhilin and Qi, Peng and Zhang, Saizheng and Bengio, Yoshua and Cohen, William and Salakhutdinov, Ruslan and Manning, Christopher D},
  booktitle={Proceedings of the 2018 conference on empirical methods in natural language processing},
  pages={2369--2380},
  year={2018}
}

@inproceedings{ho2020constructing,
  title={Constructing a multi-hop qa dataset for comprehensive evaluation of reasoning steps},
  author={Ho, Xanh and Nguyen, Anh-Khoa Duong and Sugawara, Saku and Aizawa, Akiko},
  booktitle={Proceedings of the 28th International Conference on Computational Linguistics},
  pages={6609--6625},
  year={2020}
}

@article{trivedi2022musique,
  title={MuSiQue: Multihop Questions via Single-hop Question Composition},
  author={Trivedi, Harsh and Balasubramanian, Niranjan and Khot, Tushar and Sabharwal, Ashish},
  journal={Transactions of the Association for Computational Linguistics},
  volume={10},
  pages={539--554},
  year={2022},
  publisher={MIT Press One Broadway, 12th Floor, Cambridge, Massachusetts 02142, USA~…}
}

@inproceedings{bai2025longbenchv2,
  title={Longbench v2: Towards deeper understanding and reasoning on realistic long-context multitasks},
  author={Bai, Yushi and Tu, Shangqing and Zhang, Jiajie and Peng, Hao and Wang, Xiaozhi and Lv, Xin and Cao, Shulin and Xu, Jiazheng and Hou, Lei and Dong, Yuxiao and others},
  booktitle={Proceedings of the 63rd Annual Meeting of the Association for Computational Linguistics (Volume 1: Long Papers)},
  pages={3639--3664},
  year={2025}
}

@article{zhang2023h2o,
  title={H2o: Heavy-hitter oracle for efficient generative inference of large language models},
  author={Zhang, Zhenyu and Sheng, Ying and Zhou, Tianyi and Chen, Tianlong and Zheng, Lianmin and Cai, Ruisi and Song, Zhao and Tian, Yuandong and R{\'e}, Christopher and Barrett, Clark and others},
  journal={Advances in Neural Information Processing Systems},
  volume={36},
  pages={34661--34710},
  year={2023}
}

@inproceedings{durrett2016learning,
  title={Learning-based single-document summarization with compression and anaphoricity constraints},
  author={Durrett, Greg and Berg-Kirkpatrick, Taylor and Klein, Dan},
  booktitle={Proceedings of the 54th Annual Meeting of the Association for Computational Linguistics (Volume 1: Long Papers)},
  pages={1998--2008},
  year={2016}
}

@inproceedings{lee2017endtoend,
  title={End-to-end neural coreference resolution},
  author={Lee, Kenton and He, Luheng and Lewis, Mike and Zettlemoyer, Luke},
  booktitle={Proceedings of the 2017 conference on empirical methods in natural language processing},
  pages={188--197},
  year={2017}
}

@inproceedings{maynez2020faithfulness,
  title={On faithfulness and factuality in abstractive summarization},
  author={Maynez, Joshua and Narayan, Shashi and Bohnet, Bernd and McDonald, Ryan},
  booktitle={Proceedings of the 58th annual meeting of the association for computational linguistics},
  pages={1906--1919},
  year={2020}
}

@article{lewis2020retrieval,
  title={Retrieval-augmented generation for knowledge-intensive nlp tasks},
  author={Lewis, Patrick and Perez, Ethan and Piktus, Aleksandra and Petroni, Fabio and Karpukhin, Vladimir and Goyal, Naman and K{\"u}ttler, Heinrich and Lewis, Mike and Yih, Wen-tau and Rockt{\"a}schel, Tim and others},
  journal={Advances in neural information processing systems},
  volume={33},
  pages={9459--9474},
  year={2020}
}

@misc{dubey2024llama3,
  title={The llama 3 herd of models},
  author={Grattafiori, Aaron and Dubey, Abhimanyu and Jauhri, Abhinav and Pandey, Abhinav and Kadian, Abhishek and Al-Dahle, Ahmad and Letman, Aiesha and Mathur, Akhil and Schelten, Alan and Vaughan, Alex and others},
  journal={arXiv preprint arXiv:2407.21783},
  year={2024}
}

@misc{yang2025qwen3,
  title={Qwen3 technical report},
  author={Yang, An and Li, Anfeng and Yang, Baosong and Zhang, Beichen and Hui, Binyuan and Zheng, Bo and Yu, Bowen and Gao, Chang and Huang, Chengen and Lv, Chenxu and others},
  journal={arXiv preprint arXiv:2505.09388},
  year={2025}
}

@misc{jiang2023mistral,
      title={Mistral 7B}, 
      author={Albert Q. Jiang and Alexandre Sablayrolles and Arthur Mensch and Chris Bamford and Devendra Singh Chaplot and Diego de las Casas and Florian Bressand and Gianna Lengyel and Guillaume Lample and Lucile Saulnier and Lélio Renard Lavaud and Marie-Anne Lachaux and Pierre Stock and Teven Le Scao and Thibaut Lavril and Thomas Wang and Timothée Lacroix and William El Sayed},
      year={2023},
      eprint={2310.06825},
      archivePrefix={arXiv},
      primaryClass={cs.CL},
      url={https://arxiv.org/abs/2310.06825}, 
}

@inproceedings{devlin2019bert,
  title={Bert: Pre-training of deep bidirectional transformers for language understanding},
  author={Devlin, Jacob and Chang, Ming-Wei and Lee, Kenton and Toutanova, Kristina},
  booktitle={Proceedings of the 2019 conference of the North American chapter of the association for computational linguistics: human language technologies, volume 1 (long and short papers)},
  pages={4171--4186},
  year={2019}
}

@misc{qwen3embedding2025,
  title={Qwen3 embedding: Advancing text embedding and reranking through foundation models},
  author={Zhang, Yanzhao and Li, Mingxin and Long, Dingkun and Zhang, Xin and Lin, Huan and Yang, Baosong and Xie, Pengjun and Yang, An and Liu, Dayiheng and Lin, Junyang and others},
  journal={arXiv preprint arXiv:2506.05176},
  year={2025}
}

@misc{touvron2023llama2,
  title={Llama 2: Open foundation and fine-tuned chat models},
  author={Touvron, Hugo and Martin, Louis and Stone, Kevin and Albert, Peter and Almahairi, Amjad and Babaei, Yasmine and Bashlykov, Nikolay and Batra, Soumya and Bhargava, Prajjwal and Bhosale, Shruti and others},
  journal={arXiv preprint arXiv:2307.09288},
  year={2023}
}

@inproceedings{liu2023tcrallm,
  title={Tcra-llm: Token compression retrieval augmented large language model for inference cost reduction},
  author={Liu, Junyi and Li, Liangzhi and Xiang, Tong and Wang, Bowen and Qian, Yiming},
  booktitle={Findings of the association for computational linguistics: EMNLP 2023},
  pages={9796--9810},
  year={2023}
}

@inproceedings{fei2024semantic,
  title={Extending context window of large language models via semantic compression},
  author={Fei, Weizhi and Niu, Xueyan and Zhou, Pingyi and Hou, Lu and Bai, Bo and Deng, Lei and Han, Wei},
  booktitle={Findings of the Association for Computational Linguistics: ACL 2024},
  pages={5169--5181},
  year={2024}
}

@inproceedings{shandilya2025tacorl,
  title={Taco-rl: Task aware prompt compression optimization with reinforcement learning},
  author={Shandilya, Shivam and Xia, Menglin and Ghosh, Supriyo and Jiang, Huiqiang and Zhang, Jue and Wu, Qianhui and R{\"u}hle, Victor and Rajmohan, Saravan},
  booktitle={Findings of the Association for Computational Linguistics: ACL 2025},
  pages={1582--1597},
  year={2025}
}

@article{lin2025keydensity,
  title={Prompt Compression based on Key-Information Density},
  author={Lin, Yuhao and Guo, Wenya and Zhang, Ying and Yang, Chengyi and Li, Zengxiang},
  journal={Expert Systems with Applications},
  volume={284},
  pages={127738},
  year={2025},
  publisher={Elsevier}
}

@misc{shandilya2024characterizing,
  title={Characterizing prompt compression methods for long context inference},
  author={Jha, Siddharth and Erdogan, Lutfi Eren and Kim, Sehoon and Keutzer, Kurt and Gholami, Amir},
  journal={arXiv preprint arXiv:2407.08892},
  year={2024}
}

@article{nagle2024ratdist,
  title={Fundamental limits of prompt compression: A rate-distortion framework for black-box language models},
  author={Nagle, Alliot and Girish, Adway and Bondaschi, Marco and Gastpar, Michael and Makkuva, Ashok Vardhan and Kim, Hyeji},
  journal={Advances in Neural Information Processing Systems},
  volume={37},
  pages={94934--94970},
  year={2024}
}

@inproceedings{zhang2024activationbeacon,
  title={Long context compression with activation beacon},
  author={Zhang, Peitian and Liu, Zheng and Xiao, Shitao and Shao, Ninglu and Ye, Qiwei and Dou, Zhicheng},
  booktitle={International Conference on Learning Representations},
  volume={2025},
  pages={101932--101948},
  year={2025}
}

@inproceedings{choi2024r2c,
  title={From reading to compressing: Exploring the multi-document reader for prompt compression},
  author={Choi, Eunseong and Lee, Sunkyung and Choi, Minjin and Park, Jun and Lee, Jongwuk},
  booktitle={Findings of the Association for Computational Linguistics: EMNLP 2024},
  pages={14734--14754},
  year={2024}
}

@inproceedings{deng2024giststudy,
  title={A silver bullet or a compromise for full attention? a comprehensive study of gist token-based context compression},
  author={Deng, Chenlong and Zhang, Zhisong and Mao, Kelong and Li, Shuaiyi and Huang, Xinting and Yu, Dong and Dou, Zhicheng},
  booktitle={Proceedings of the 63rd Annual Meeting of the Association for Computational Linguistics (Volume 1: Long Papers)},
  pages={4861--4879},
  year={2025}
}

@inproceedings{li2025500x,
  title={500xcompressor: Generalized prompt compression for large language models},
  author={Li, Zongqian and Su, Yixuan and Collier, Nigel},
  booktitle={Proceedings of the 63rd Annual Meeting of the Association for Computational Linguistics (Volume 1: Long Papers)},
  pages={25081--25091},
  year={2025}
}

@misc{liu2025infopreserve,
  title={Understanding and improving information preservation in prompt compression for llms},
  author={{\L}ajewska, Weronika and Hardalov, Momchil and Aina, Laura and John, Neha Anna and Su, Hang and M{\`a}rquez, Llu{\'\i}s},
  journal={arXiv preprint arXiv:2503.19114},
  year={2025}
}

@inproceedings{tang2025perception,
  title={Perception compressor: A training-free prompt compression framework in long context scenarios},
  author={Tang, Jiwei and Xu, Jin and Lu, Tingwei and Zhang, Zhicheng and YimingZhao, YimingZhao and LinHai, LinHai and Zheng, Hai-Tao},
  booktitle={Findings of the Association for Computational Linguistics: NAACL 2025},
  pages={4093--4108},
  year={2025}
}

@article{radford2019language,
  title={Language models are unsupervised multitask learners},
  author={Radford, Alec and Wu, Jeffrey and Child, Rewon and Luan, David and Amodei, Dario and Sutskever, Ilya and others},
  journal={OpenAI blog},
  volume={1},
  number={8},
  pages={9},
  year={2019}
}

@misc{openai2026models,
  author       = {{OpenAI}},
  title        = {{Models}},
  year         = {2026},
  howpublished = {OpenAI API documentation},
  note         = {Accessed July 18, 2026. Available at \url{https://developers.openai.com/api/docs/models}}
}

@misc{zhipu2026glm52,
  author       = {{Zhipu AI}},
  title        = {{GLM-5.2}},
  year         = {2026},
  howpublished = {Zhipu AI open documentation},
  note         = {Accessed July 18, 2026. Available at \url{https://docs.bigmodel.cn/cn/guide/models/text/glm-5.2}}
}

\clearpage
\appendix

\section{Robustness to Experimental Choices}
\label{app:robustness}

The main diagnostic depends on the \textsc{Beaver} embedding model, the content-word-overlap threshold, and the compression ratio. This appendix varies each choice and shows that the diagnosis remains stable across all variations.

\paragraph{Embedding scorer.} \textsc{Beaver} scores chunks by cosine similarity to the query under an embedding model; the main text uses the released Qwen3-0.6B embeddings configuration. Substituting the GPT-2 checkpoint~\citep{radford2019language} leaves the HotpotQA dangling rate essentially unchanged (Table~\ref{tab:robustness}, $n{=}300$, ratio $0.30$). The GPT-2 encoder changes $\rho_{\mathrm{d}}$ by at most $0.9$ percentage points. The diagnosis is therefore robust across both \textsc{Beaver} embedding encoders.

\paragraph{Overlap threshold.} The unified cross-compressor metric uses the sentence-level predicate $\operatorname{Keep}_{\theta}$ defined in Section~\ref{sec:diagnosis}; paragraph-level prevalence uses the separate paragraph instantiation described there. Sweeping $\theta$ from $0.3$ to $0.7$ (Figure~\ref{fig:threshold}, Table~\ref{tab:threshold-full}) keeps dangling substantial for both a chunk-level and a token-level compressor across the entire range, and at every threshold up to $0.6$, including the setting most lenient to token fragments, LLMLingua-2 dangles at least as much as \textsc{Beaver}, whereas a metric biased against fragments would show the reverse. So the cross-signal agreement in Section~\ref{sec:paradigm} is not an artifact of one lenient cutoff. The token-level curve is nonmonotonic by construction: at strict thresholds a partially retained supporting sentence fails the overlap test and the case moves from dangling to full evidence loss rather than to safety, whereas \textsc{Beaver}'s coherent blocks keep clearing the bar.

\begin{table}[h]
\centering
\small
\setlength{\tabcolsep}{3pt}
\begin{tabular}{lccccc}
\toprule
Overlap threshold & $0.3$ & $0.4$ & $0.5$ & $0.6$ & $0.7$ \\
\midrule
LLMLingua-2 (token) & $28.3$ & $43.5$ & $56.0$ & $57.6$ & $36.4$ \\
\textsc{Beaver} (chunk) & $19.6$ & $25.0$ & $32.1$ & $36.4$ & $40.2$ \\
\bottomrule
\end{tabular}
\caption{Dangling rate (\%) across content-word overlap retention thresholds on the same $184$ bridge examples as Figure~\ref{fig:paradigm} at compression ratio $0.30$. The $0.5$ column matches Figure~\ref{fig:paradigm}.}
\label{tab:threshold-full}
\vspace{-15pt}
\end{table}

\begin{table}[h]
\centering
\scriptsize
\setlength{\tabcolsep}{3pt}
\begin{tabular}{@{}lll@{}}
\toprule
Check & Variation & Outcome \\
\midrule
Embedding scorer & Qwen3-0.6B embeddings $\to$ GPT-2 & $+0.9$ points \\
Overlap threshold & $0.3$ to $0.7$               & substantial throughout \\
Compression ratio & $0.70$ to $0.20$             & monotonic increase \\
\bottomrule
\end{tabular}
\caption{Robustness of the dangling diagnostic to its three main free choices (HotpotQA, $n{=}300$, ratio $0.30$ unless swept). Embedding shifts are measured in percentage points relative to the released Qwen3-0.6B embedding setup.}
\label{tab:robustness}
\vspace{-15pt}
\end{table}

\begin{figure}[h]\centering
\includegraphics[width=\columnwidth]{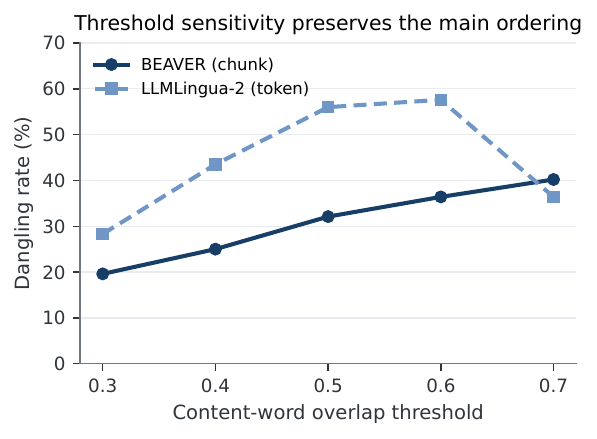}
\caption{Dangling rate across content-word-overlap thresholds for $184$ bridge examples (Figure~\ref{fig:paradigm}; ratio $0.30$).}
\vspace{-15pt}
\label{fig:threshold}\end{figure}

\begin{figure*}[h]
\centering
\includegraphics[width=0.9\textwidth]{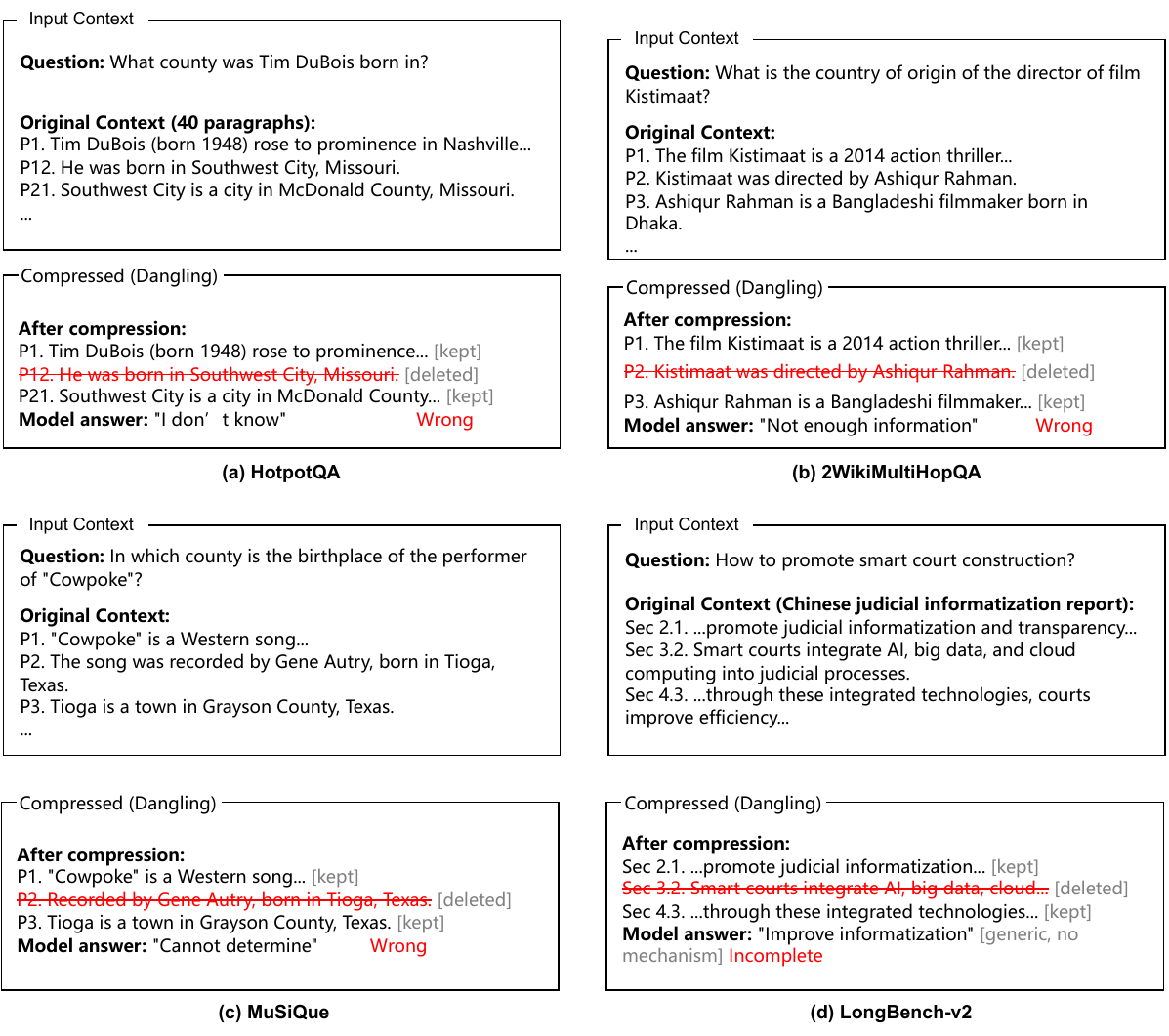}
\caption{Referential dangling examples from HotpotQA, 2WikiMultiHopQA, MuSiQue, and LongBench-v2 Single-Document QA. Each panel shows the original context and the compressed output.}
\label{fig:gallery}
\vspace{-18pt}
\end{figure*}

\paragraph{Compression ratio.} The compression ratio analysis (Figure~\ref{fig:dose}a) sweeps the ratio from $0.70$ to $0.20$ and finds the HotpotQA dangling rate rising monotonically with no qualitative threshold or reversal. The finding is therefore not limited to a single operating ratio. Table~\ref{tab:robustness} summarizes all three checks.

\section{Manual Audit and Qualitative Examples}
\label{app:gallery}

We present one representative dangling case from each dataset to illustrate the ``answer present but bridge deleted'' mechanism. All examples are real cases from the experimental data. Each panel shows the original context and the compressed output with referential dangling. Among the $38$ confirmed cases in the manual audit, the omitted paragraph contained the dataset answer string in $26$, the entity definition in $10$, and another dependency in $2$. The two false positives involved an answer paragraph that was sufficient on its own and an omitted paragraph that supplied only a generic modifier.

\section{Long-Document Dangling by Subdomain}
\label{app:longdoc}

Table~\ref{tab:longdoc} breaks down the LongBench-v2 Single-Document QA diagnostic (Section~\ref{sec:diagnosis}, $n{=}80$, ratio $0.30$) by subdomain. The per-kept-sentence dangling rate is consistent across all seven subdomains ($25.1\%$ to $36.5\%$), and every document in every subdomain has at least one dangling reference. Academic text has the highest rate and also contains dense terminology, frequent cross-section references, and the longest retained-sentence counts. This pattern is consistent with the association between dependency span and dangling reported in Section~\ref{sec:diagnosis}. No subdomain falls below $25\%$, suggesting that the phenomenon is not confined to a single document type.

\begin{table}[h]
\centering
\small
\begin{tabular}{lccc}
\toprule
Subdomain & $n$ & Mean rate & Affected docs \\
\midrule
Academic & $13$ & $36.5\%$ & $100\%$ \\
Literary & $12$ & $34.4\%$ & $100\%$ \\
Financial & $12$ & $32.3\%$ & $100\%$ \\
Legal & $8$ & $29.1\%$ & $100\%$ \\
Detective & $15$ & $27.3\%$ & $100\%$ \\
Event ordering & $11$ & $27.0\%$ & $100\%$ \\
Governmental & $9$ & $25.1\%$ & $100\%$ \\
\midrule
\textbf{All} & $80$ & $\mathbf{30.5\%}$ & $\mathbf{100\%}$ \\
\bottomrule
\end{tabular}
\caption{Referential dangling on LongBench-v2 Single-Document QA by subdomain (\textsc{Beaver}, ratio $0.30$, $n{=}80$ documents). ``Mean rate'' is the per-document average fraction of retained sentences that are dangling, macro-averaged over documents. ``Affected docs'' is the fraction of documents with at least one dangling reference.}
\label{tab:longdoc}
\vspace{-15pt}
\end{table}

\section{Protocol for Fixed-Budget Content Reselection}
\label{app:reselection}

The comparison in Section~\ref{sec:reselection} measures the accuracy change when a missing supporting paragraph replaces lower scoring paragraphs that are not annotated as supporting the answer, without increasing the token budget.

\paragraph{Three conditions.} Each dangling example is evaluated under three conditions. Base is the compressor's original dangling output at target compression ratio $0.30$. Reselected reinserts one omitted supporting paragraph and removes paragraphs not annotated as supporting the answer whose combined length is at least as large. Full support retains both annotated supporting paragraphs while removing two paragraphs not annotated as supporting the answer. We use Base and Reselected for the fixed-budget comparison, while Full support provides a reference for complete evidence retention.

\paragraph{Token budget enforcement.} We tokenize all paragraphs with spaCy, which is also used during compressor preprocessing, and track cumulative token counts. For Reselected, we rank omitted supporting paragraphs by their isolated salience score under the compressor and select the highest scoring paragraph. We then remove the lowest scoring paragraphs not annotated as supporting the answer until their combined token count matches or exceeds that of the reinserted paragraph. Contexts under Reselected are shorter than those under Base on average by $1.5\%$ on HotpotQA, $1.6\%$ on 2WikiMultiHopQA, and $9.6\%$ on MuSiQue.

\paragraph{Dataset specifics.} HotpotQA uses bridge examples with two annotated supporting paragraphs. Panel 1 uses a dangling subset of $80$ examples, while Panel 2 uses a separate evaluation set of $200$ examples. 2WikiMultiHopQA uses compositional examples with exactly two annotated supporting paragraphs and a dangling subset of $72$ examples. MuSiQue uses paragraph-level supporting facts with exactly two annotated supporting paragraphs per example in the first $200$ examples of the development split, yielding a dangling subset of $102$ examples. The dangling subset uses directional content-word overlap at threshold $0.5$ (Section~\ref{sec:diagnosis}): the answer paragraph survives but its paired definition paragraph does not.

\paragraph{Statistical test.} Significance is assessed with the two-sided exact McNemar test on the $2{\times}2$ contingency table of fixes, where the answer under Base is incorrect and the comparison answer is correct, and breaks, where the answer under Base is correct and the comparison answer is incorrect. All three datasets yield $p<10^{-4}$ for the comparison between Base and Reselected on the dangling subset. Base versus Full support yields $p<0.01$ for all four downstream LLMs on the $200$-example HotpotQA set.

\section{Compressor Configurations}
\label{app:compressor-configs}

The six compressors of Section~\ref{sec:paradigm} all compress the same $184$ HotpotQA bridge examples to ratio $0.30$ (keep $30\%$ of tokens). We use official implementations where available; training-free methods require no adaptation.

\begin{table}[t]
\centering
\scriptsize
\setlength{\tabcolsep}{3pt}
\renewcommand{\arraystretch}{1.12}
\begin{tabular}{@{}p{0.34\columnwidth}p{0.56\columnwidth}@{}}
\toprule
\makecell[l]{Role and reported name} & \makecell[l]{Official checkpoint or API identifier} \\
\midrule
\makecell[l]{Beaver scorer\\Qwen3-0.6B embeddings} & \url{Qwen/Qwen3-0.6B} \\
\makecell[l]{Robustness scorer\\GPT-2} & \url{openai-community/gpt2} \\
\makecell[l]{DAC proxy\\Qwen3 0.6B} & \url{Qwen/Qwen3-0.6B} \\
\midrule
\makecell[l]{Downstream QA\\Qwen3 4B; Qwen3 8B} & \url{Qwen/Qwen3-4B}; \url{Qwen/Qwen3-8B} \\
\makecell[l]{Downstream QA\\Llama 3.1 8B; Mistral 7B} & \makecell[l]{\url{meta-llama/Llama-3.1-8B-Instruct}\\\url{mistralai/Mistral-7B-Instruct-v0.3}} \\
\midrule
\makecell[l]{Self-information /\\perplexity proxy\\Llama 2 7B} & \url{meta-llama/Llama-2-7b-hf} \\
\midrule
\makecell[l]{Compression /\\restoration classifiers} & \makecell[l]{\url{microsoft/llmlingua-2-xlm-roberta-large-}\\\url{meetingbank};\\\url{google-bert/bert-base-uncased}} \\
\midrule
\makecell[l]{Proprietary downstream QA\\GPT-5.5; GLM-5.2} & API IDs: \url{gpt-5.5}; \url{glm-5.2} \\
\bottomrule
\end{tabular}
\caption{Official checkpoint and API identifiers. Display names are the shorthand used in the paper; exact identifiers are shown for reproducibility.}
\label{tab:checkpoints}
\vspace{-15pt}
\end{table}

\paragraph{\textsc{Beaver} (embedding similarity, query-aware).} Official repository \texttt{github.com/JusperLee/BEAVER}, coherent-block selection. We use the released configuration with the \url{Qwen/Qwen3-0.6B} checkpoint as the embedding scorer. Each document is segmented into $64$-token pages; the compressor scores pages by cosine similarity between the query embedding and the page's inverse document frequency weighted token embedding average, then selects top-$k$ pages to meet the target ratio. Hardware: NVIDIA A100 80GB.

\paragraph{LLMLingua-2 (trained token classifier, query-agnostic).} Official repository \url{https://github.com/microsoft/LLMLingua}, checkpoint \url{microsoft/llmlingua-2-xlm-roberta-large-meetingbank}. Token-level binary classifier over a sliding $512$-token window; tokens are kept if the classifier score exceeds a threshold calibrated to the target ratio. No query input.

\paragraph{Selective-Context (self-information, query-agnostic).} Reproduced from \citet{li2023selective}. Each token's self-information is computed as $-\log p(\text{token}\mid\text{prefix})$ under the \url{meta-llama/Llama-2-7b-hf} proxy model~\citep{touvron2023llama2}; tokens with self-information below a calibrated threshold are dropped. No query input.

\paragraph{PartPrompt (syntactic parse tree, query-agnostic).} Official repository. Constituency parse tree built with Berkeley Neural Parser; each node (phrase) scores by syntactic salience (depth and span); a knapsack solver based on dynamic programming selects a subset of nodes covering the target ratio. The selected nodes' token spans are concatenated in document order. No query input.

\paragraph{LongLLMLingua (perplexity, query-aware).} Official repository \texttt{microsoft/LLMLingua}. Scores sentences by perplexity under the proxy model conditioned on the query and preceding context; keeps lowest-perplexity (most ``expected'') sentences up to the budget. Uses the \url{meta-llama/Llama-2-7b-hf} proxy model~\citep{touvron2023llama2}.

\paragraph{DAC (attention, query-agnostic).} Official implementation, method \texttt{dynamic\_attn\_ppl}, fusion \texttt{additive} with $\alpha{=}0.8$. Token-level importance derived from attention weights during a single forward pass of the \url{Qwen/Qwen3-0.6B} proxy model~\citep{yang2025qwen3} over the context; tokens below the importance threshold are dropped. Parameter \texttt{compress\_ratio} (the drop fraction) set to $0.70$ to keep $30\%$. We run DAC with a memory-efficient attention accumulator that is numerically identical to its original scoring, allowing the full $184$-example bridge set to fit in memory.

All compressors run single-threaded on an A100 80GB for consistency.

\section{Implementation Details for Automatic Context Restoration}
\label{app:detector}

The following engineering details document the automatic restoration pipeline of Section~\ref{sec:method} for reproducibility.

\begin{table}[h]
\centering
\small
\setlength{\tabcolsep}{3pt}
\begin{tabular}{llccc}
\toprule
Downstream LLM & \makecell{Candidate\\source} & Base & Restored & $p$ \\
\midrule
Qwen3-8B & First mention & 0.567 & \textbf{0.613} & 0.022 \\
Mistral-7B & First mention & 0.455 & \textbf{0.520} & 0.012 \\
Llama-3.1-8B & First mention & 0.587 & 0.600 & 0.60 \\
Llama-3.1-8B & Hybrid & 0.587 & 0.610 & 0.17 \\
\bottomrule
\end{tabular}
\caption{Automatic restoration results with the classifier fixed at $K{=}3$. The evaluation uses $300$ HotpotQA examples, except for Mistral-7B, which uses $200$. Base is \textsc{Beaver} at compression ratio $0.30$, and Restored has an average ratio of $0.31$. The reported $p$ values use paired McNemar tests.}
\label{tab:method-full}
\vspace{-15pt}
\end{table}

\begin{figure}[h]\centering
\includegraphics[width=\columnwidth]{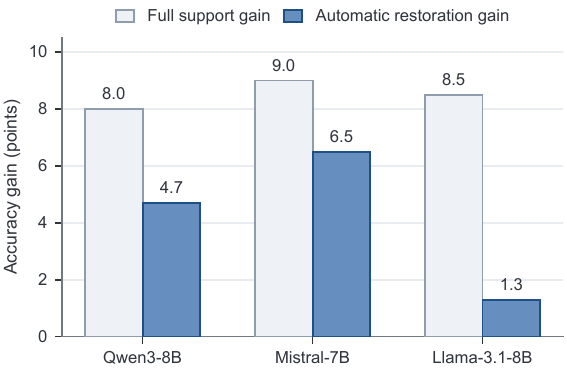}
\caption{Accuracy gains from full support and first-mention automatic restoration on HotpotQA (\textsc{Beaver} at ratio $0.30$, $K{=}3$). Full support uses $200$ examples; restoration uses $300$ for Qwen3-8B and Llama-3.1-8B and $200$ for Mistral-7B.}
\label{fig:repair}\end{figure}

\begin{table}[h]
\centering
\small
\setlength{\tabcolsep}{3pt}
\begin{tabular}{lcc}
\toprule
Feature & Fixed (23) & Failed (107) \\
\midrule
Sentences added, mean $\pm$ SD & 2.13 $\pm$ 1.08 & 1.79 $\pm$ 1.17 \\
Sentences added, median & 3.0 & 2.0 \\
\midrule
\multicolumn{3}{c}{McNemar: 23 fixes, 9 breaks, $p=0.022$} \\
\bottomrule
\end{tabular}
\caption{Restoration statistics when the \textsc{Beaver} baseline was incorrect (HotpotQA, $n{=}300$, compression ratio $0.30$; downstream Qwen3-8B). SD denotes standard deviation.}
\label{tab:repair_patterns}
\end{table}

\begin{table}[h]
\centering
\footnotesize
\setlength{\tabcolsep}{3pt}
\begin{tabular}{lcc}
\toprule
Condition & Accuracy [95\% CI] & $\Delta$ \\
\midrule
Base compressor & $0.567$ {\footnotesize[.51,.62]} & \NA \\
Random insertion, $m$ sentences & $0.587$ {\footnotesize[.53,.64]} & $+2.0$ \\
Targeted restoration, $m$ sentences & $\mathbf{0.613}$ {\footnotesize[.55,.67]} & $+4.7$ \\
\bottomrule
\end{tabular}
\caption{Matched addition control on HotpotQA with Qwen3-8B ($n{=}300$). Random insertion and targeted restoration add the same number of sentences per example ($m$: mean $1.81$, median $2$, interquartile range $[1,3]$; approximately $40$ tokens; compression ratio $0.30$ to $0.31$; $K{=}3$). Brackets report bootstrap $95\%$ confidence intervals.}
\label{tab:b1control}
\end{table}

\begin{table}[h]
\centering
\footnotesize
\setlength{\tabcolsep}{3pt}
\begin{tabular}{lcc}
\toprule
Compressor (output type) & Gain & McNemar $p$ \\
\midrule
\textsc{Beaver} (coherent chunks) & $+4.7$ & $p{=}0.022$ \\
PartPrompt (parse tree spans) & $+3.2$ & $0.30$ \\
Selective-Context (self-information) & $+1.0$ & $0.80$ \\
LLMLingua-2 (token fragments) & $+1.0$ & $0.75$ \\
\bottomrule
\end{tabular}
\caption{Transfer of one restoration configuration across four compressor outputs on HotpotQA with downstream Qwen3-8B ($n{\approx}150$ to $300$).}
\label{tab:boundary}
\vspace{-15pt}
\end{table}

\paragraph{Output granularity.} We evaluate automatic restoration on \textsc{Beaver} because its coherent multi-sentence blocks permit controlled changes at sentence boundaries and keep restored sentences interpretable. Compressors with token-level outputs produce fragmented text in which inserting complete supporting sentences may alter the original selection objective and complicate comparison. Appendix H reports transfer across four compressor outputs, including LLMLingua-2. The classifier does not take compressor scores or identities as input, although its training distribution and integration strategy may affect transfer across output granularities.

\paragraph{Classifier architecture and training.} The classifier is \texttt{bert-base-uncased} (about $110$M parameters) with a binary classification head over a retained sentence, a candidate sentence, and the question, separated by \texttt{[SEP]} tokens. The maximum sequence length is $256$. We fine-tune the classifier for three epochs with AdamW using a learning rate of $2\times10^{-5}$, a batch size of $32$, a linear schedule with $10\%$ warmup, and fp16, and retain the checkpoint with the highest development F1. The resulting development F1 is $0.82$, with a precision of $0.84$.

\paragraph{Training pair construction.} Positive pairs consist of a retained sentence and an omitted annotated supporting sentence that share an entity. Hard negatives also share an entity but do not satisfy this positive-pair rule, while easy negatives share no entity. Positive pairs account for $1.9\%$ of all candidate pairs ($3026/157887$), so we downsample negatives to $1.5\times$ the number of positives, yielding $7{,}565$ rebalanced pairs with $40\%$ positives. Pairs are split by example identifier so that no example crosses the training and development boundary. The source contexts come from the HotpotQA training split and are disjoint from the evaluation contexts used in all downstream experiments.

\paragraph{Number of restored sentences ($K$).} At inference, we add the top-$K$ candidates by classifier confidence. In the Qwen3-8B sweep, $K{=}2$ gives $+3.7$ points ($p{=}0.08$), $K{=}3$ gives $+4.7$ points ($p{=}0.02$), and adding all candidates gives $+4.3$ points ($p{=}0.03$). We use $K{=}3$, the setting with the largest observed gain, which adds about $1.8$ sentences per example.

\section{Automatic Context Restoration Results and Controls}
\label{app:repair-table}

\paragraph{Downstream results.} Table~\ref{tab:method-full} reports the accuracies and significance tests for automatic restoration, and Figure~\ref{fig:repair} compares the restoration gain with the gain obtained when both supporting paragraphs are retained.

\paragraph{Restoration outcomes.} Table~\ref{tab:repair_patterns} compares cases corrected by automatic restoration with those that remain incorrect among the $130$ examples for which the \textsc{Beaver} baseline is incorrect. The two groups receive similar numbers of restored sentences. These measurements show that restoration size alone does not distinguish the outcomes.

\paragraph{Matched addition control.} Table~\ref{tab:b1control} reports the control experiment of Section~\ref{sec:method}. Random insertion and targeted restoration add the same number of sentences per example.

\section{Transfer Across Compressor Outputs}
\label{app:boundary}

Table~\ref{tab:boundary} applies the same restoration configuration to each compressor output with Qwen3-8B as the downstream model on HotpotQA. The classifier and candidate generator are calibrated on \textsc{Beaver}'s retained sentence distribution and HotpotQA first-mention structure, so the experiment evaluates transfer of one fixed configuration rather than configurations optimized separately for each compressor.

\end{document}